\documentclass[sigconf, nonacm]{acmart}
\usepackage{microtype}
\usepackage{graphicx}
\usepackage{subcaption}
\usepackage{tikz}
\usetikzlibrary{shapes, arrows.meta, positioning, fit, backgrounds, calc, shadows}
\usepackage{multirow}
\usepackage{array}
\usepackage{float}
\usepackage{algorithmic}
\usepackage{algorithm}
\usepackage{makecell}

\newcommand\vldbdoi{XX.XX/XXX.XX}

\newcommand\vldbavailabilityurl{URL_TO_YOUR_ARTIFACTS}

\begin{document}
\title{Revisiting Privacy Amplification by Subsampling in Selective Release DPSGD}

\author{Xiaobo Huang}
\affiliation{%
  \institution{Guangdong Provincial Key Laboratory of IRADS, Beijing Normal-Hong Kong Baptist University}
  \streetaddress{P.O. Box 1212}
  \city{Zhuhai}
  \state{China}
  \postcode{519087}
}

\author{Fang Xie}
\orcid{}
\affiliation{%
  \institution{Guangdong Provincial Key Laboratory of IRADS, Beijing Normal-Hong Kong Baptist University}
  \streetaddress{P.O. Box 1212}
  \city{Zhuhai}
  \state{China}
  \postcode{519087}
}

\begin{abstract}
Machine learning's reliance on sensitive data necessitates privacy-preserving techniques like Differentially Private Stochastic Gradient Descent (DPSGD). However, DPSGD suffers from substantial utility degradation and slow convergence due to gradient clipping and noise injection. Prior works have attempted to improve DPSGD from various perspectives; notably, the Differentially Private Selective Update and Release (DPSUR) algorithm has achieved remarkable model utility. However, the privacy accounting in DPSUR overlooks the variation in sampling probability introduced by the selective release mechanism, which compromises the rigor of its privacy guarantees. To address these limitations, we re-evaluate the privacy analysis of the selective release mechanism and propose a novel algorithm: Differentially Private Selective Release based on Clipped Gradients (DPSR-CG). Through a rigorous, newly derived privacy analysis and extensive experiments on multiple datasets (MNIST, CIFAR-10, IMDB, and FMNIST), we demonstrate that our DPSR-CG mechanism maintains strict privacy guarantees while achieving exceptional model performance.
\end{abstract}

\maketitle


\ifdefempty{\vldbavailabilityurl}{}{
\vspace{.3cm}
\begingroup\small\noindent\raggedright\textbf{Artifact Availability:}\\
{The source code, data, and/or other artifacts have been made available at \url{ https://github.com/FangXieLab/DPSR-CB}.}
\endgroup
}

\section{Introduction}
Deep learning has achieved immense success in real-world tasks, yet it entails significant privacy risks. Adversaries can invert models to recover sensitive training data, prompting strict regulations such as GDPR, CCPA, and HIPAA \citep{protection2018general, cummings2018role, annas2003hipaa}.

Among privacy-preserving paradigms, Differentially Private Stochastic Gradient Descent (DPSGD) \citep{abadi2016deep} has become the de facto baseline. Despite its popularity, DPSGD suffers from inherent limitations: (1) gradient clipping discards valuable signals from distinct samples, introducing significant bias; (2) Gaussian noise injection introduces high variance that destabilizes convergence; (3) Poisson sampling increases computational complexity; and (4) the algorithm is highly sensitive to hyperparameters, particularly the clipping threshold.

Prior works have attempted to mitigate these issues through sampling optimization \citep{wei2022dpis}, data-centric approaches \citep{tramer2020differentially}, and activation function selection \citep{papernot2021tempered}. Notably, \citet{fu2023dpsur} introduced the Selective Update and Release (DPSUR) mechanism, significantly narrowing the utility gap. However, DPSUR incurs high computational costs due to frequent validation steps and still suffers from gradient misalignment caused by clipping. More importantly, the privacy analysis in the previously proposed DPSUR does not fully account for the variation in sampling probability introduced by the selective release mechanism, which may compromise the accuracy of their privacy accounting.

To address these challenges, we propose a novel framework: Differentially Private Selective Release based on Clipped Gradients (DPSR-CG). Instead of the restrictive minimal clipping in DPSUR, we utilize gradient scaling \citep{bu2023convergence} to manage update sensitivity. We also apply this scaling methodology to the gradient clipping bias, which is used to determine whether to release the current update. Our main contributions are:
\begin{itemize}
\item We identify a gap in the privacy accounting of the existing selective release algorithm and provide a rigorous privacy analysis to address it.
\item We propose a new differentially private selective release algorithm based on clipping bias, which strictly maintains privacy guarantees.
\item We provide extensive evaluations on MNIST, CIFAR-10, IMDB, and FMNIST. Results demonstrate that DPSR-CG outperforms state-of-the-art solutions in model accuracy.
\end{itemize}

\section{Preliminary}
\label{sec:preliminary}

In this section, we introduce the fundamental concepts of Differential Privacy, Rényi Differential Privacy, and the standard framework for DPSGD. Furthermore, we discuss the theoretical basis for the Gaussian mechanism with selective release, which serves as the foundation for our proposed method.

\subsection{Differential Privacy}

Differential Privacy (DP) provides a rigorous mathematical framework for quantifying data privacy. It ensures that the output of a randomized algorithm does not significantly depend on the presence or absence of any single individual's data in the input dataset~\cite{dwork2014algorithmic}.

\begin{definition}[($\epsilon, \delta$)-Differential Privacy]
A randomized mechanism $\mathcal{M}:\mathcal{X}^{n}\rightarrow\mathcal{Y}$ satisfies $(\epsilon, \delta)$-differential privacy if for any two neighboring datasets $D, D^{\prime}\in\mathcal{X}^{n}$ that differ by a single entry, and for any subset of outputs $S\subseteq\mathcal{Y}$, it holds that:
\begin{equation*}
Pr[\mathcal{M}(D)\in S]\le e^{\epsilon}\cdot Pr[\mathcal{M}(D^{\prime})\in S]+\delta,
\end{equation*}
where $\epsilon>0$ represents the privacy budget, and $\delta\ge0$ is the failure probability.
\end{definition}

To achieve DP, noise is typically added proportional to the sensitivity of the function. For a function $f:\mathcal{X}^{n}\rightarrow\mathbb{R}^{d}$, the $l_{2}$-sensitivity is defined as $\Delta_{2}f=\sup_{D,D^{\prime}}||f(D)-f(D^{\prime})||_{2}$.

\subsection{Rényi Differential Privacy}
\label{pre_rdp}
Rényi Differential Privacy (RDP) is a natural relaxation of pure differential privacy based on Rényi divergence. It is particularly advantageous for iterative algorithms like DPSGD, as it provides tighter bounds for sequential composition~\cite{mironov2017renyi}.

\begin{definition}[($\alpha, \epsilon$)-RDP]
A randomized mechanism $\mathcal{M}$ satisfies $(\alpha, \epsilon)$-RDP if for any two neighboring datasets $D, D^{\prime}$, the Rényi divergence of order $\alpha > 1$ between their output distributions satisfies:
\begin{equation*}
D_{\alpha}(\mathcal{M}(D)||\mathcal{M}(D^{\prime}))\le\epsilon.
\end{equation*}
\end{definition}

For a function with $l_2$-sensitivity $\Delta$, the standard Gaussian mechanism $\mathcal{N}(0,\sigma^{2}\Delta^{2}I)$ satisfies $(\alpha, \frac{\alpha}{2\sigma^{2}})$-RDP. In DPSGD, this mechanism is coupled with random data subsampling, forming the Sampled Gaussian Mechanism (SGM). 

\begin{definition}[RDP of Sampled Gaussian Mechanism]
\label{def:rdp_sgm}
Let $SG_{q,\sigma}$ denote the SGM with a subsampling rate $q$ and noise multiplier $\sigma$ applied to a sensitivity-1 query. $SG_{q,\sigma}$ satisfies $(\alpha, R)$-RDP, where:
\begin{equation*}
R \le \frac{1}{\alpha-1} \ln [\max(A_{\alpha}(q,\sigma), B_{\alpha}(q,\sigma))].
\end{equation*}
Here, $A_{\alpha}$ and $B_{\alpha}$ bound the divergence between the baseline and mixed distributions. Following standard procedures~\cite{abadi2016deep, mironov2017renyi}, the dominant moment $A_{\alpha}(q,\sigma)$ can be analytically evaluated via the binomial expansion:
\begin{equation*}
A_{\alpha}(q,\sigma) = \sum_{k=0}^{\alpha}\binom{\alpha}{k}(1-q)^{\alpha-k}q^{k}\exp\left(\frac{k^{2}-k}{2\sigma^{2}}\right).
\end{equation*}
\end{definition}

To provide standard differential privacy guarantees, the accumulated RDP budget across all iterations is ultimately converted to $(\epsilon, \delta)$-DP using the following lemma.

\begin{lemma}[Conversion from RDP to DP~\cite{mironov2017renyi}]
\label{lem:rdp_to_dp}
If a randomized mechanism $f: \mathcal{D} \rightarrow \mathbb{R}$ satisfies $(\alpha, R)$-RDP, it strictly satisfies $(\epsilon, \delta)$-DP for any $0<\delta<1$, where:
\begin{equation*}
\epsilon = R + \frac{\ln(1/\delta) + (\alpha-1)\ln(1-1/\alpha) - \ln(\alpha)}{\alpha-1}.
\end{equation*}
\end{lemma}

\subsection{Deep Learning with Differential Privacy}

Differentially Private Stochastic Gradient Descent (DPSGD)~\cite{abadi2016deep} is the standard algorithm for training deep neural networks with privacy guarantees. It modifies the standard SGD by clipping per-sample gradients and adding Gaussian noise.

In each iteration $t$, a mini-batch $\mathcal{B}_{t}$ is sampled. For each sample $x_{i}\in\mathcal{B}_{t}$, the gradient $g_{t}(x_{i})=\nabla_{\theta}\mathcal{L}(\theta,x_{i})$ is computed and clipped to bound the sensitivity:
\begin{equation*}
\bar{g}_t(x_i) = g_t(x_i) / \max \left(1, \frac{||g_t(x_i)||_2}{C} \right),
\end{equation*}
where $C$ is the clipping threshold. Gaussian noise is then added to the sum of clipped gradients:
\begin{equation*}
\tilde{g}_{t}=\frac{1}{|\mathcal{B}_{t}|}\left(\sum_{x_{i}\in\mathcal{B}_{t}}\bar{g}_{t}(x_{i})+\mathcal{N}(0,\sigma^{2}C^{2}I)\right).
\end{equation*}
The model parameters are updated using $\tilde{g}_{t}$.

\subsection{Gaussian Mechanism with Selective Release}

In advanced frameworks (e.g., DPSUR~\cite{fu2023dpsur}), mechanisms may choose to release the noisy output only if it falls within a specific range. It has been proven that such selective release mechanisms maintain the same RDP guarantee as the standard Gaussian mechanism, provided the selection decision is based on privatized values.

\begin{theorem}[RDP of Selective Gaussian Mechanism~\cite{fu2023dpsur}]
\label{thm:selective_release}
Let $\mathcal{M}$ be a mechanism that releases $f(D) + \mathcal{N}(0, \sigma^2)$ only if the output falls within a truncation interval (e.g., $(a, b]$). This mechanism satisfies $(\alpha, \alpha/2\sigma^2)$-RDP, which is identical to the untruncated Gaussian mechanism.
\end{theorem}

This property is critical for our proposed DPSR-CG, as it allows the algorithm to discard updates with high clipping bias without incurring additional privacy costs for the rejected steps.

\section{Calibration of Sampling Probability in Selectively Release Paradigm}
Before detailing the privacy analysis of our proposed algorithm DPSR-CG, we first address certain theoretical limitations in the privacy accounting of the existing selective release algorithm, DPSUR. Subsequently, we present a mathematically rigorous privacy accounting framework to formalize the guarantees of our mechanism.

\subsection{Preview of the Previous Selective Release Paradigm}

\begin{figure}[htbp]
    \centering
    \resizebox{0.9\linewidth}{!}{
    \begin{tikzpicture}[
        node distance=0.6cm and 1.2cm,
        font=\sffamily,
        block/.style={
            rectangle, rounded corners=3pt, draw=gray!80, thick, fill=white,
            minimum width=4.5cm,
            minimum height=0.8cm, align=center,
            drop shadow={opacity=0.15, shadow xshift=1pt, shadow yshift=-1pt}
        },
        decision/.style={
            diamond, aspect=2, draw=orange!80, thick, fill=orange!5,
            minimum width=3cm, inner sep=1pt, align=center,
            drop shadow={opacity=0.15}
        },
        arrow/.style={-{Stealth[length=2.5mm]}, thick, rounded corners=5pt, color=black!70},
        label/.style={font=\footnotesize\itshape, color=blue!60!black, midway, right, xshift=2pt},
        status/.style={font=\bfseries\footnotesize, pos=0.5}
    ]

        \pgfdeclarelayer{background}
        \pgfsetlayers{background,main}
        
        \node (in) [block, fill=gray!5, draw=gray!50] {Input: $T$, $Z$, $t=1$, $w_0=\text{Initial}()$};

        \node (s1) [block, below=0.5cm of in] {
            \textbf{Step 1:}\\[3pt]
            Get new model $w_{\text{new}}$ by DPSGD
        };
        
        \node (s2) [block, below=0.5cm of s1] {
            \textbf{Step 2:}\\[3pt]
            Get loss $J(w_{\text{new}})$ and $J(w_{t-1})$
        };
        
        \node (s3) [block, below=0.5cm of s2] {
            \textbf{Step 3:}\\[3pt]
            $\Delta E = J(w_{\text{new}}) - J(w_{t-1})$
        };
        
        \node (s4) [block, below=0.5cm of s3] {
            \textbf{Step 4:}\\[3pt]
            $\widetilde{\Delta E} \xleftarrow{\text{DP processing}} \Delta E$
        };
        
        \node (dec1) [decision, below=0.8cm of s4] {$\widetilde{\Delta E} < Z$?}; 
        
        \node (s5yes) [block, fill=green!5, draw=green!60!black, minimum width=2.8cm, below left=0.6cm and 0.2cm of dec1] {
            \textbf{Accept Update}\\[3pt]
            $w_t = w_{\text{new}}$\\
            $t = t + 1$
        };
        
        \node (s5no)  [block, fill=red!5, draw=red!60!black, minimum width=2.8cm, below right=0.6cm and 0.2cm of dec1] {
            \textbf{Reject Update}\\[3pt]
            $w_t = w_{t-1}$
        };

        \begin{pgfonlayer}{background}
            \node [fit=(s2)(s3)(s4), fill=orange!5, rounded corners, draw=blue!30, dashed, inner sep=6pt] (step234Group) {};
        \end{pgfonlayer}

        \begin{pgfonlayer}{background}
            \node [fit=(dec1)(s5yes)(s5no), fill=orange!5, rounded corners, draw=blue!30, dashed, inner sep=12pt] (step5Group) {};
        \end{pgfonlayer}
        
        \node (dec2) [decision, draw=orange!80, fill=orange!10, below=0.5cm of step5Group] {$t < T$?}; 
        
        \node (out) [block, fill=orange!5, draw=orange!50, below=0.5cm of dec2] {Output model $w_T$};

        \draw [arrow] (in) -- (s1);
        \draw [arrow] (s1) -- (s2);
        \draw [arrow] (s2) -- (s3);
        \draw [arrow] (s3) -- (s4);
        \draw [arrow] (s4) -- (dec1);
        
        \draw [arrow] (dec1.west) -| node[above, status, color=red!60, xshift=8pt] {Yes} (s5yes.north);
        \draw [arrow] (dec1.east) -| node[above, status, color=green!60, xshift=-8pt] {No} (s5no.north);

        \coordinate (merge) at ($(s5yes.south)!0.5!(s5no.south) - (0, 0.4cm)$);
        \draw [thick, rounded corners=5pt, color=black!70] (s5yes.south) |- (merge);
        \draw [thick, rounded corners=5pt, color=black!70] (s5no.south) |- (merge);
        \draw [arrow] (merge) -- (dec2.north);

        \coordinate (loop) at ($(s1.west) - (2.5cm, 0)$); 
        \draw [arrow] (dec2.west) -- node[above, status, color=red!70] {Yes} (dec2.west -| loop) -- (loop) -- (s1.west);

        \draw [arrow] (dec2.south) -- node[right, status, color=green!70] {No} (out.north);

    \end{tikzpicture}
    }
    \caption{Workflow of DPSUR.}
    \label{fig:workflow_dpsur}
\end{figure}
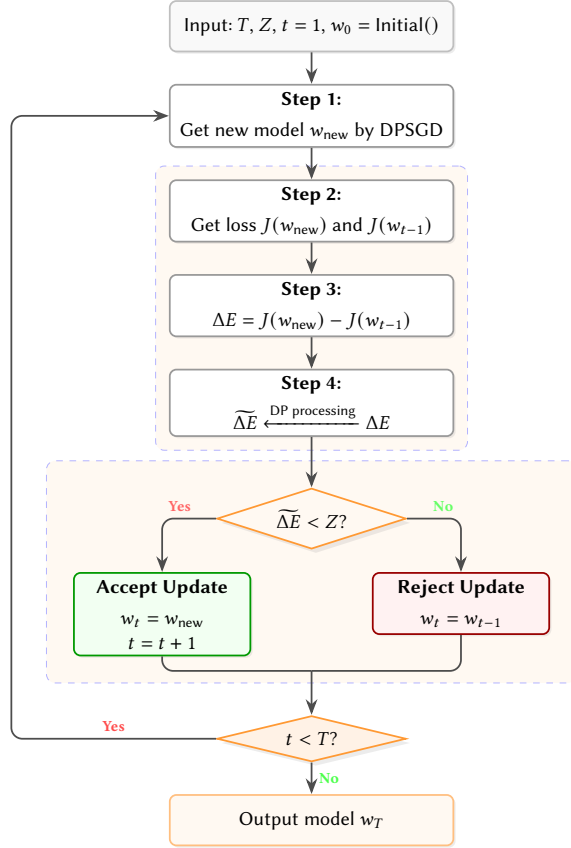

Figure \ref{fig:workflow_dpsur} illustrates the operational pipeline of the DPSUR framework. The algorithm takes several initial inputs: the total number of iterations $T$, the current count of accepted updates $t$, a predefined acceptance threshold $Z$, and an initial model state $w_0$. The procedure follows these sequential steps to output the final model:

\begin{itemize}
    \item \textbf{Step 1:} During each iteration, a data batch is drawn from the training dataset using Poisson sampling. This batch is then utilized to compute a candidate intermediate model, denoted as $w_{new}$, via the standard DPSGD optimizer.
    
    \item \textbf{Step 2:} A distinct validation batch, $\mathcal{B}_v$, is independently sampled from the training data. Both the newly proposed model $w_{new}$ and the previously accepted model $w_{t-1}$ are evaluated on $\mathcal{B}_v$ to determine their respective empirical loss values, $J(w_{new})$ and $J(w_{t-1})$.
    
    \item \textbf{Step 3:} To assess whether the candidate model improves performance, the loss difference is calculated as\\
    $\Delta E = J(w_{new}) - J(w_{t-1})$.
    
    \item \textbf{Step 4:} To enforce differential privacy guarantees during the validation phase, the calculated difference $\Delta E$ is clipped to bound its sensitivity, followed by the injection of calibrated noise. This yields the privatized loss difference, $\widetilde{\Delta E}$.
    
    \item \textbf{Step 5:} The algorithm compares the noisy difference $\widetilde{\Delta E}$ against the predefined threshold $Z$. If $\widetilde{\Delta E} < Z$, the candidate model $w_{new}$ is accepted as the new model state, and the accepted update counter $t$ is incremented by 1. Otherwise, the update is discarded, and the algorithm retains the last accepted model, $w_{t-1}$.
\end{itemize}

\subsection{Review of Privacy Accounting in DPSUR}
Before analyzing the theoretical limitations of existing methods, it is essential to review the baseline privacy accounting mechanism employed by the selective release paradigm, specifically DPSUR. DPSUR optimizes the overall privacy budget by leveraging the Gaussian Mechanism with Selective Release. The core premise of their accounting framework is that the training trajectory only leaks information when a model update is officially accepted and applied.

Specifically, DPSUR tracks the privacy loss through two independent phases for every \textit{accepted} iteration:
\begin{itemize}
    \item \textbf{Training Phase:} The computation of the intermediate candidate model weights using standard DPSGD.
    \item \textbf{Validation Phase:} The evaluation of the clipped loss difference $\widetilde{\Delta E}$ on a separate validation batch to determine update acceptance.
\end{itemize}

According to their theoretical analysis, rejected updates—along with their corresponding intermediate weights and noisy validation metrics—are discarded and kept strictly internal. Therefore, DPSUR assumes that only the $T$ successfully accepted updates consume the privacy budget. For each of these $T$ accepted steps, the privacy leakage is quantified using the standard Sampled Gaussian Mechanism (SGM) under the assumption of a uniform Poisson sampling probability $q = B/N$. The total privacy loss is computed by linearly accumulating the Rényi Differential Privacy (RDP) divergences of these $T$ independent steps from both phases, which is subsequently converted to the standard $(\epsilon, \delta)$-DP guarantee via the composition theorem.

\subsection{Limitations in the Privacy Accounting of DPSUR}

In the selective release paradigm, the algorithm first conducts Poisson sampling to form a data batch, computes the gradients, and injects Gaussian noise. Before updating the model, a validation step is performed to verify whether the validation loss (or clipping bias) decreases. The model update is accepted only if this noisy condition is met. 

Fundamentally, this process satisfies differential privacy (DP) provided the noise level is calibrated to the appropriate privacy parameters (i.e., $\epsilon$ and $\delta$). While the DPSUR framework is highly novel and inspiring, its privacy analysis relies on the assumption that the sampling probability of each data record remains uniformly $q$ across all accepted updates. Furthermore, the analysis does not strictly bound the contribution of a single sample to the validation loss. Instead, it estimates the sensitivity based on the bounded domain of the aggregated validation loss, which does not align with the standard definition of sensitivity in DP.

The assumption of a constant sampling probability warrants closer examination because the decision to accept or reject an update acts as a data-dependent filter. If a batch is formed using an initial sampling probability $q$, the probability that a specific sample is included in an \textit{accepted} batch deviates from $q$. It implicitly transforms into a conditional probability, which we denote as $q_{sup} = \mathbb{P}(z \in \mathcal{B} \mid v> C - \Delta E(\mathcal{B}), \mathcal{F}_{t-1})$, where $\mathcal{F}_{t-1}$ is filtration, which encapsulates all previously accepted model states and historical release decisions. In the subsequent sections, we demonstrate how to rigorously construct and bound this effective sampling probability.

Furthermore, the threat model and background assumptions in previous literature remain ambiguous. Prior works claim that rejecting an update incurs no privacy cost, which implies that the validation metric used for the decision is not released to the adversary (otherwise, the privacy leakage from rejected updates must be accounted for). Paradoxically, their privacy analysis explicitly incorporates the cost of releasing the validation metric, directly contradicting their initial premise.

While existing selective release frameworks,  introduce highly inspiring methodologies for improving model utility in DPSGD, we identify two theoretical nuances that require closer examination to ensure strict and logically consistent privacy guarantees. To address these, we propose the following refinements:

\begin{itemize}
    \item \textbf{Re-evaluating Privacy Accounting:} Prior analyses assume a uniform sampling probability $q$ across accepted updates, overlooking that selective release acts as a data-dependent filter. Furthermore, estimating sensitivity via unbounded validation loss risks underestimating privacy leakage. \\
    \textit{Our Refinement:} We rigorously bound the worst-case effective sampling probability, $q_{sup} = \mathbb{P}(z \in \mathcal{B} \mid v> C - \Delta E(\mathcal{B}), \mathcal{F}_{t-1})$. Coupled with our dual-clipping mechanism, this enforces a strict per-sample sensitivity bound, preventing unconstrained probability inflation.

    \item \textbf{Clarifying the Threat Model:} Previous literature suggests rejected updates incur no privacy cost, yet their accounting occasionally incorporates the release of validation metrics, creating a theoretical misalignment. \\
    \textit{Our Refinement:} We distinguish our context from the traditional Sparse Vector Technique (SVT). Because our practical threat model \textit{never} releases rejection signals or internal noisy metrics to the adversary, unapplied updates genuinely incur zero privacy cost, ensuring a logically consistent accounting baseline.
\end{itemize}

\subsection{Conditional Privacy Loss under Adaptive Filtration}
\label{sec:conditional_privacy}
In standard DPSGD, iterative updates rely on uniform Poisson sampling, allowing for straightforward sequential composition. However, the selective release mechanism in DPSR-CG acts as a data-dependent filter. Consequently, the probability of an update being accepted at step $t$ is no longer independent; it is conditionally dependent on the filtration $\mathcal{F}_{t-1}$, which encapsulates all previously accepted model states and historical release decisions.

To rigorously account for the privacy loss under this adaptive mechanism, we must evaluate the worst-case privacy leakage over all possible historical trajectories in $\mathcal{F}_{t-1}$. Let $\mathcal{B}$ denote a batch sampled via baseline Poisson sampling with rate $q = |\mathcal{B}|/N$. The update is conditionally released if the clipped validation deviation satisfies $\Delta E(\mathcal{B}) + v > C$, where $v \sim \mathcal{N}(0, \sigma_v^2)$.

According to Bayes' theorem, for any sample $z$, its conditional sampling probability $q_{cond}(z)$ in a successfully accepted update, given the historical filtration $\mathcal{F}_{t-1}$, is explicitly expanded as:
\begin{align*}
    q_{cond}(z) &= \mathbb{P}(z \in \mathcal{B} \mid v > C - \Delta E(\mathcal{B}), \mathcal{F}_{t-1}) \\
    &= \frac{\mathbb{P}(v> C - \Delta E(\mathcal{B}) \mid z \in \mathcal{B}, \mathcal{F}_{t-1}) \cdot \mathbb{P}(z \in \mathcal{B} \mid \mathcal{F}_{t-1})}{\mathbb{P}(v> C - \Delta E(\mathcal{B}) \mid \mathcal{F}_{t-1})} \\
\intertext{Since $\mathbb{P}(z \in \mathcal{B} \mid \mathcal{F}_{t-1})$ represents the sampling probability in our implementation, thus $\mathbb{P}(z \in \mathcal{B} \mid \mathcal{F}_{t-1}) = q$,}
    q_{cond}(z) &= q \cdot \frac{\mathbb{P}(v> C - \Delta E(\mathcal{B}) \mid z \in \mathcal{B}, \mathcal{F}_{t-1})}{\mathbb{P}(v> C - \Delta E(\mathcal{B}) \mid \mathcal{F}_{t-1})}.
\end{align*}

The ratio $\rho = \frac{\mathbb{P}(v> C - \Delta E(\mathcal{B}) \mid z \in \mathcal{B}, \mathcal{F}_{t-1})}{\mathbb{P}(v> C - \Delta E(\mathcal{B}) \mid \mathcal{F}_{t-1})}$ quantifies the relative probability inflation caused by the inclusion of sample $z$. To decouple the privacy accounting from the actual, highly complex historical trajectory, we must bound the worst-case effective sampling probability $q_{sup} = q \cdot \rho_{max}$ by evaluating the supremum of $\rho$.

To mathematically decouple the privacy accounting from the intricate dependencies of the historical trajectory $\mathcal{F}_{t-1}$, we parameterize the dynamic threshold gap as a scalar distance $x = C - \Delta E(\mathcal{B})$. This distance represents the exact magnitude of noise $v$ required to trigger a successful release. Given the dual-clipping bounds, $x$ is inherently constrained within $[-C_{max}, C_{max}]$. By abstracting the evaluation metric into this state-independent variable $x$, we can transform the evaluation of the complex conditional probability into a tractable worst-case supremum problem, as formalized in the following theorem:

\begin{theorem}[Filtration-Independent Probability Inflation Bound]
\label{thm:privacy_loss}
Let $x = C - \Delta E(\mathcal{B}) \in [-C_{max}, C_{max}]$ denote the threshold distance required to release an update. The worst-case probability inflation ratio $\rho_{max}$ is strictly bounded by:$$\rho_{max} = \min \left( \frac{1 - \Phi_{\sigma_v}(C_{max} - s)}{1 - \Phi_{\sigma_v}(C_{max})}, \frac{1}{q} \right),$$where $q$ is the baseline Poisson sampling rate, $\Phi_{\sigma_v}$ is the cumulative distribution function of $\mathcal{N}(0, \sigma_v^2)$, and $s$ denotes the maximum sensitivity of the evaluation metric.
\end{theorem}

\begin{proof}
Given an arbitrary neighbor dataset $\mathcal{D}$ and $\mathcal{D}'= \mathcal{D} \cup \{z\}$, the conditional acceptance probability relies on the noise mechanism. Let $x$ represent the distance to the threshold required by the noise mechanism for batch $\mathcal{B}$ drawn from $\mathcal{D}$, conditioned on the filtration $\mathcal{F}_{t-1}$. Assuming an unbounded right-tail verification region, the conditional release probability is:
$$
\mathbb{P}(v> x \mid \mathcal{F}_{t-1}) = 1 - \Phi_{\sigma_v}(x).
$$
In the worst-case scenario, the inclusion of sample $z$ shifts the evaluation metric by the maximum possible sensitivity $s$, making the release condition easier to satisfy:
$$
\mathbb{P}(v> x \mid z \in \mathcal{B}, \mathcal{F}_{t-1}) = 1 - \Phi_{\sigma_v}(x - s).
$$
Substituting these terms into the probability inflation ratio $\rho$ derived from Bayes' theorem, we have:
$$
\rho = \frac{1 - \Phi_{\sigma_v}(x - s)}{1 - \Phi_{\sigma_v}(x)}.
$$
To decouple the privacy leakage from the filtration $\mathcal{F}_{t-1}$, we take the supremum of $\rho$ over all possible evaluation states $x$:
$$
\rho_{max} = \sup_{x} \frac{1 - \Phi_{\sigma_v}(x - s)}{1 - \Phi_{\sigma_v}(x)}.
$$
Since conditional probabilities must satisfy $\mathbb{P}(v> x \mid z \in \mathcal{B}, \mathcal{F}_{t-1}) \cdot q \leq 1$, we inherently have the restriction $\rho \leq \frac{1}{q}$. Thus:
$$
\rho_{max} = \min \left( \sup_{x} \frac{1 - \Phi_{\sigma_v}(x - s)}{1 - \Phi_{\sigma_v}(x)}, \frac{1}{q} \right).
$$
A fundamental property of the Gaussian distribution establishes that the ratio of its right-tail probabilities, $\frac{1 - \Phi(x - s)}{1 - \Phi(x)}$, is strictly monotonically increasing with respect to $x$ for any $s > 0$. Therefore, the supremum is obtained at the maximum possible value of $x$.

Crucially, the threshold evaluation in DPSR-CG is constrained by the gradient clipping and bias bounds. As $x \in [-C_{max}, C_{max}]$, $x$ cannot approach infinity and is strictly blocked by the physical upper limit $C_{max}$. By substituting $x = C_{max}$ into the supremum, we mathematically isolate the maximum inflation from the filtration $\mathcal{F}_{t-1}$ and obtain the exact bounded constant:
$$
\rho_{max} = \min \left( \frac{1 - \Phi_{\sigma_v}(C_{max} - s)}{1 - \Phi_{\sigma_v}(C_{max})}, \frac{1}{q} \right).
$$
This data-independent constant strictly bounds the relative privacy cost per effective step regardless of the model's history, satisfying the prerequisite for standard adaptive sequential composition. We conclude the proof.
\end{proof}

\section{METHODOLOGY}
In this section, we present our proposed framework, DPSR-CG. We begin with a high-level overview, followed by a detailed introduction of its two key components: the gradient and bias scaling mechanism, and the bias-based threshold mechanism.

\subsection{DP Training Framework with Selective Release Based on Clipping Gradients}
Fig.~\ref{fig:my_workflow} shows the DPSR-CG workflow. In each step towards $T$ effective updates, the algorithm samples a batch to compute private gradients and estimate the gradient bias via DPSGD-Global. The bias is then scaled and perturbed with Gaussian noise for privacy. Crucially, the model update is accepted only if this noisy bias variation satisfies a defined threshold; otherwise, it is rejected.

\begin{figure}[htbp]
    \centering
    \resizebox{1.0\linewidth}{!}{
    \begin{tikzpicture}[
        node distance=0.6cm and 1.2cm,
        font=\sffamily,
        block/.style={
            rectangle, rounded corners=3pt, draw=gray!80, thick, fill=white,
            minimum width=4.5cm,
            minimum height=0.8cm, align=center,
            drop shadow={opacity=0.15, shadow xshift=1pt, shadow yshift=-1pt}
        },
        decision/.style={
            diamond, aspect=2, draw=orange!80, thick, fill=orange!5,
            minimum width=3cm, inner sep=1pt, align=center,
            drop shadow={opacity=0.15}
        },
        arrow/.style={-{Stealth[length=2.5mm]}, thick, rounded corners=5pt, color=black!70},
        label/.style={font=\footnotesize\itshape, color=blue!60!black, midway, right, xshift=2pt},
        status/.style={font=\bfseries\footnotesize, pos=0.5}
    ]

        \pgfdeclarelayer{background}
        \pgfsetlayers{background,main}
        
        \node (in) [block, fill=gray!5, draw=gray!50] {Input: Data $\mathcal{D}$, $T$, $t=1$, $w_0=\text{Initial}()$};

        \node (s1) [block, below=0.5cm of in] {
            \textbf{Step 1: Private Trainer (DPSGD-Global)}\\[3pt]
            $w_{new} \leftarrow w_{t-1} - \eta \cdot (\tilde{g} + \mathcal{N}_t)$
        };
        
        \node (s2) [block, below=0.5cm of s1] {
            \textbf{Step 2: Clipped Gradients Estimation}\\[3pt]
            $\Delta E = \sum_{i=0}^{n} E_{t-1,i} - \sum_{i=0}^{n} E_{t,i}$
        };
        
        \node (s3) [block, below=0.5cm of s2] {
            \textbf{Step 3: Privacy Protection}\\[3pt]
            $\widetilde{\Delta E} \leftarrow \text{Clip}(\Delta E) + \mathcal{N}_v$
        };
        
        \node (dec1) [decision, below=1.0cm of s3] {$\widetilde{\Delta E} > Threshold$?}; 
        
        \node (s5yes) [block, fill=green!5, draw=green!60, minimum width=2.8cm, below left=0.6cm and 0.2cm of dec1] {
            \textbf{Accept Update}\\[3pt]
            $w_t = w_{new}$\\
            $t = t + 1$
        };
        
        \node (s5no)  [block, fill=red!5, draw=red!60, minimum width=2.8cm, below right=0.6cm and 0.2cm of dec1] {
            \textbf{Reject Update}\\[3pt]
            $w_t = w_{t-1}$
        };

        \begin{pgfonlayer}{background}
            \node [fit=(s2)(s3), fill=orange!5, rounded corners, draw=black!30, dashed, inner sep=10pt] (valGroup) {};
            \node [right=5pt of valGroup, font=\small\bfseries\sffamily, color=orange!60!black, rotate=-90, anchor=north] (validate) {};
        \end{pgfonlayer}

        \begin{pgfonlayer}{background}
            \node [fit=(dec1)(s5yes)(s5no), fill=orange!5, rounded corners, draw=black!30, dashed, inner sep=12pt] (step5Group) {};
        \end{pgfonlayer}
        
        \node (dec2) [decision, draw=orange!80, fill=orange!10, below=0.5cm of step5Group] {$t < T$?}; 
        
        \node (out) [block, fill=orange!5, draw=orange!50, below=0.5cm of dec2] {Output model $w_T$};

        \draw [arrow] (in) -- (s1);
        \draw [arrow] (s1) -- node[label] {Candidate $w_{new}$} (s2);
        \draw [arrow] (s2) -- (s3);
        \draw [arrow] (s3) -- (dec1);
        
        \draw [arrow] (dec1.west) -| node[above, status, color=red!60, xshift=8pt] {Yes} (s5yes.north);
        \draw [arrow] (dec1.east) -| node[above, status, color=green!60, xshift=-8pt] {No} (s5no.north);
        
        \coordinate (merge) at ($(s5yes.south)!0.5!(s5no.south) - (0, 0.4cm)$);
        \draw [thick, rounded corners=5pt, color=black!70] (s5yes.south) |- (merge);
        \draw [thick, rounded corners=5pt, color=black!70] (s5no.south) |- (merge);
        \draw [arrow] (merge) -- (dec2.north);
        
        \coordinate (loop) at ($(s1.west) - (2.5cm, 0)$); 
        \draw [arrow] (dec2.west) -- node[above, status, color=red!70] {Yes} (dec2.west -| loop) -- (loop) -- (s1.west);

        \draw [arrow] (dec2.south) -- node[right, status, color=green!70] {No} (out.north);

    \end{tikzpicture}
    }
    \caption{Workflow of the DPSR-CG. The Validator Module estimates gradient bias under privacy protection.}
    \label{fig:my_workflow}
\end{figure}
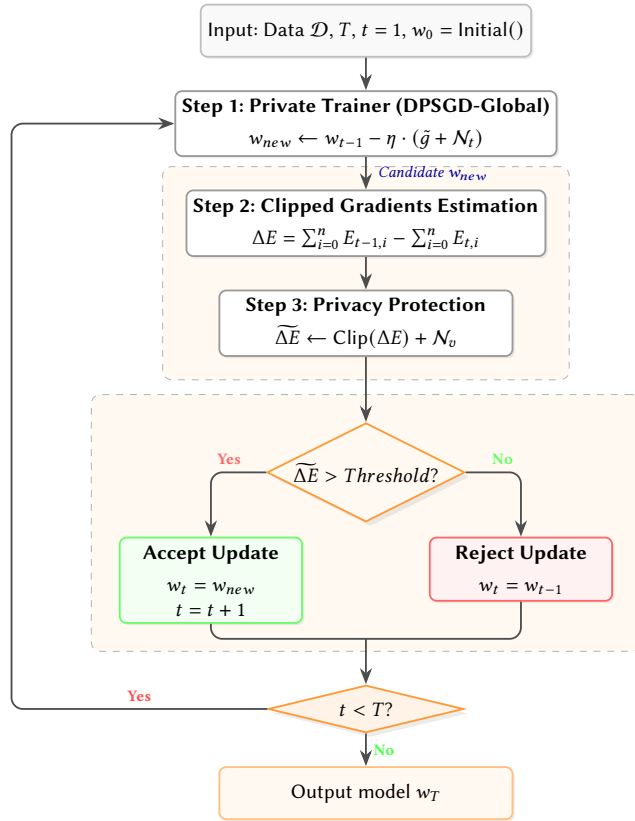

\subsection{Motivation}
The primary motivation of DPSR-CG is to efficiently identify and reject detrimental ("toxic") model updates without alternating the neural network between training and validation modes. While prior selective release frameworks (e.g., DPSUR \cite{fu2023dpsur}) rely on validation loss to evaluate update quality, evaluating a separate validation batch at each iteration incurs non-trivial computational overhead. Furthermore, we argue that the true capability of the selective release paradigm remains obscured due to the loose and unconstrained privacy accounting in existing formulations. 

To address these challenges, we introduce a framework grounded in two core design departures:
\begin{itemize}
    \item \textbf{Gradient Scaling via DPSGD-Global}: Rather than employing the restrictive minimal clipping adopted by DPSUR, we integrate the gradient scaling mechanism proposed in DPSGD-Global \cite{bu2023convergence}. This preserves the relative directional alignment of gradients while minimizing optimization distortion caused by loose scaling bounds.
    \item \textbf{Intrinsic Metric Filtration}: We substitute the computationally expensive validation loss with an intrinsic training metric—the clipping bias—thereby eliminating validation latency entirely.
\end{itemize}

\subsection{Gradients and Clipped Gradients Scaling Mechanism}
\label{sec:grad_scal}
The gradients and clipped gradients scaling mechanism builds upon the framework by \citet{bu2023convergence}, simultaneously establishing scaling and clipping bounds. These approach limits per-sample gradient sensitivity while preserving relative magnitudes and alignment. Algorithm \ref{alg:gradient_scale} outlines the workflow.

\begin{algorithm}[htbp]
\caption{Scaling Mechanism}
\label{alg:gradient_scale}
\begin{algorithmic}
    \renewcommand{\algorithmicrequire}{\textbf{Input:}}
    \renewcommand{\algorithmicensure}{\textbf{Output:}}
    \REQUIRE A batch of sensitive information $g$, Hard bounds $C$, Scaling bound $S$
    \ENSURE The information after being scaled $\bar{g}$
    \FOR{$g_i \in g$}
        \IF{$||g_i||_2 > S$}
            \STATE $\bar{g}_i \leftarrow g_i / \frac{||g_i||_2}{C}$
        \ELSE
            \STATE $\bar{g}_i \leftarrow g_i / \frac{S}{C}$
        \ENDIF
    \ENDFOR
    \STATE \textbf{return} $\bar{g}$
\end{algorithmic}
\end{algorithm}
\subsection{Threshold Mechanism}
\label{sec:threshold_mech}

Unlike DPSUR \cite{fu2023dpsur}, which uses validation loss to filter updates, DPSR-CG employs clipped gradients—the gradients norm, for which exceeding the scaling bound—as the metric. We accumulate this clipped gradients ($\bar{E}$) and reject the update if the summation of the clipped gradients in current batch $\bar{E}_t$ exceeds the previous $\bar{E}_{t-1}$.\\

\noindent\textbf{Intuition behind Clipping Bias as a Metric.} 
Standard DPSGD introduces gradient clipping to bound individual sensitivity. When a sampled mini-batch contains a high proportion of out-of-distribution samples or exhibits massive gradient variance, individual gradient norms will drastically exceed the scaling bound $S_g$, generating a substantial aggregated clipping bias. Such batches typically yield unstable, high-variance update vectors that can compromise convergence. Therefore, monitoring the evolution of the clipping bias $\Delta E = E_{t-1} - E_t$ serves as a direct, zero-overhead indicator: a positive $\Delta E$ signifies that the current batch alignment harmonizes with the ongoing optimization trajectory, whereas a negative value flags a potentially toxic update driven by gradient outliers.\\

Previous works clipped aggregated validation loss differences tightly (eq. \ref{eq:loss_clip}), often causing beneficial updates to be rejected due to noise
\begin{equation*}
\label{eq:loss_clip}
    \Delta \bar{E} = \text{clip}\left(\sum_{i=0}^n \Delta E_i\right) + \mathcal{N}(0,\sigma^2),
\end{equation*}
where $\Delta E_i = E_{t,i} - E_{t-1,i}$. To address this, we limit sensitivity by clipping per-sample bias variations: 
\begin{equation*}
 \Delta \bar{E} = \text{clip}_2\left(\sum_{i=0}^n \text{clip}_1(\Delta E_{i})\right) + \mathcal{N}(0,\sigma^2),
\end{equation*}
where $\Delta E_{i}$ is the clipped gradients difference between current batch and last batch. The purpose of the first clip ($\text{clip}_1$) is to limit the sensitivity of $\Delta E_{i}$, and we apply a second clip ($\text{clip}_2$) to ensure that the privacy protection.

To see this, consider the upper bound of the effective sampling rate $q_{sup}$ when the validation noise is unbounded (i.e., $x \to \infty$). For a differing sample that causes a maximal bias shift $s$, the inflation of the sampling rate is given by the ratio of the tail probabilities:
\begin{equation*}
    q_{sup} = min(q \cdot \frac{1 - \Phi_{\sigma_v}(C_{\text{max}} - s)}{1 - \Phi_{\sigma_v}(C_{\text{max}})} , \frac{1}{q}).
\end{equation*}

As $C_{\text{max}} \to \infty$, both the numerator and the denominator approach $0$, creating an indeterminate form. To evaluate this limit, we apply L'H\^opital's rule. Taking the derivative of the cumulative distribution function yields the probability density function $\phi_{\sigma_v}(x) = \frac{1}{\sqrt{2\pi}\sigma_v} \exp(-\frac{x^2}{2\sigma_v^2})$:
\begin{align*}
    \lim_{C_{\text{max}} \to \infty} \frac{1 - \Phi_{\sigma_v}(C_{\text{max}} - s)}{1 - \Phi_{\sigma_v}(C_{\text{max}})} 
    &= \lim_{C_{\text{max}} \to \infty} \frac{-\phi_{\sigma_v}(C_{\text{max}} - s)}{-\phi_{\sigma_v}(C_{\text{max}})} \nonumber \\
    &= \lim_{C_{\text{max}} \to \infty} \frac{\exp\left(-\frac{(C_{\text{max}} - s)^2}{2\sigma_v^2}\right)}{\exp\left(-\frac{C_{\text{max}}^2}{2\sigma_v^2}\right)}.
\end{align*}

By merging the exponents, we obtain:
\begin{align*}
    &= \lim_{C_{\text{max}} \to \infty} \exp\left( \frac{C_{\text{max}}^2 - (C_{\text{max}}^2 - 2C_{\text{max}}s + s^2)}{2\sigma_v^2} \right) \nonumber \\
    &= \lim_{C_{\text{max}} \to \infty} \exp\left( \frac{2C_{\text{max}}s - s^2}{2\sigma_v^2} \right).
\end{align*}

Since the sensitivity $s > 0$, the term $2C_{\text{max}}s$ grows unboundedly as $C_{\text{max}} \to \infty$. Consequently:
\begin{equation*}
    \lim_{C_{\text{max}} \to \infty} \rho_{max} = min(\infty, \frac{1}{q}) = \frac{1}{q},
\end{equation*}
and 
\begin{equation*}
    q_{sup} = \frac{1}{q} \cdot q = 1.
\end{equation*}
This divergence explicitly proves that without a strict physical truncation boundary, the inflation factor of the sampling probability approaches infinity for outlier batches.

\noindent \textbf{Selection of hyperparameter $\beta$}.
Fig. \ref{fig:accept_IMDB} and Fig. \ref{fig:accept_FMNIST} shows the good updates and bad updates in the acceptance set and rejection set, given different $\beta$ when training the models using IMDB dataset and FMNIST dataset under $\epsilon=3, \delta=10^{-5}$, respectively. The results show that a threshold of $\beta=3$ empirically provides a good balance, effectively allowing minor violations to accelerate the overall training process.
\begin{figure*}[h]
    \centering
    \includegraphics[width=0.7\textwidth]{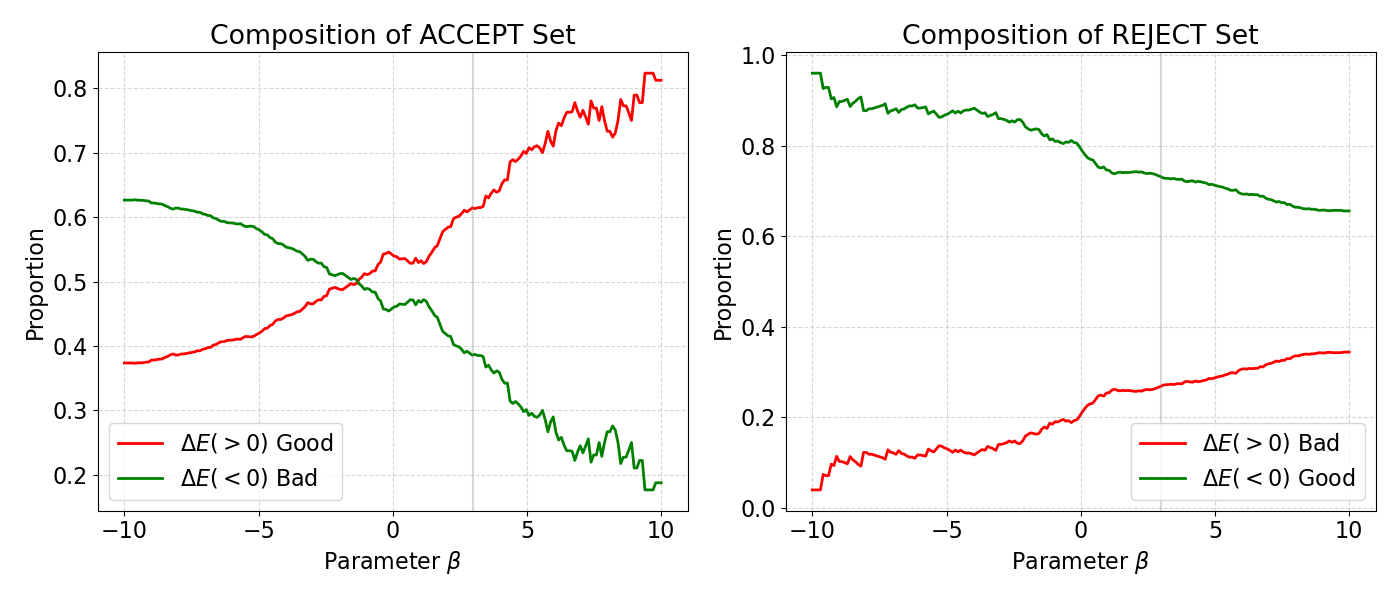}
    \caption{Proportion of Good/Bad Accept and Rejection in Accept and Rejection set of IMDB dataset}
    \label{fig:accept_IMDB}
\end{figure*}

\begin{figure*}[h]
    \centering
    \includegraphics[width=0.7\textwidth]{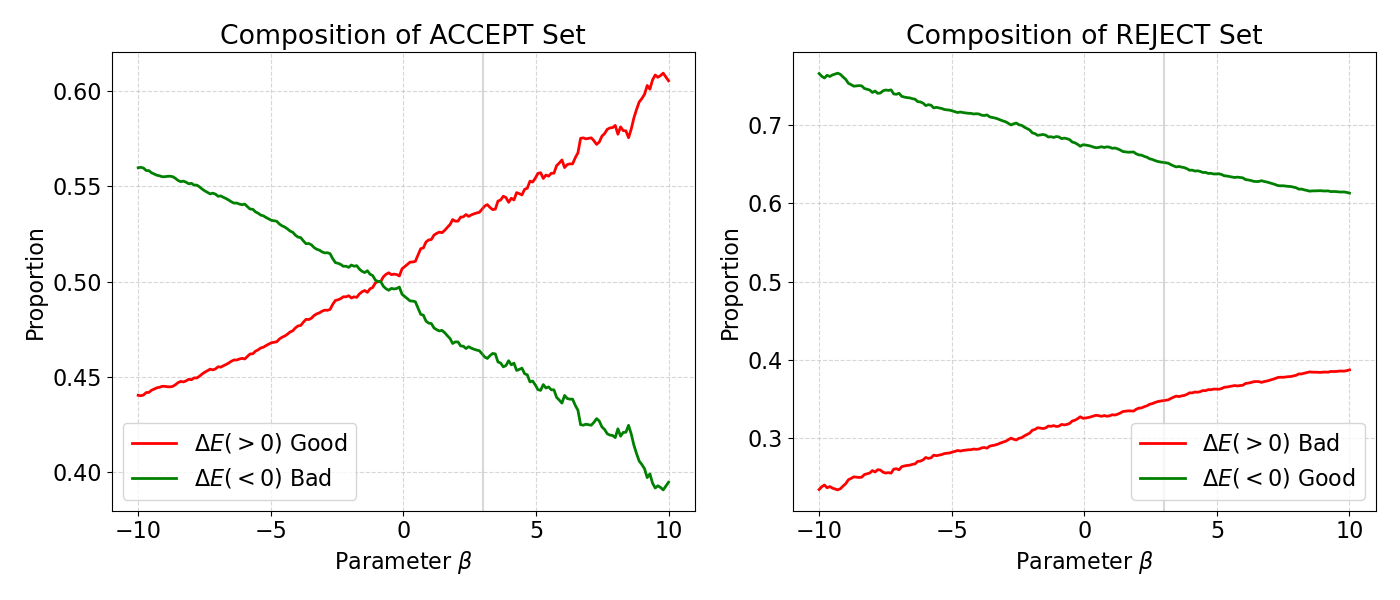}
    \caption{Proportion of Good/Bad Accept and Rejection in Accept and Rejection set of FMNIST dataset}
    \label{fig:accept_FMNIST}
\end{figure*}

\subsection{DPSR-CG: Algorithm Details}
We now detail the DPSR-CG framework Algorithm \ref{alg:dpsr_cb}.
\subsubsection*{1. Initialization and Sampling Adjustment}
The algorithm begins by initializing the model parameters $w_0$ and setting the iteration counter $t = 0$. Before the main training loop, it recalculates the actual sampling probability ($q_{sup}$) and the effective batch size ($B_t$). These adjustments utilize the cumulative distribution function ($\Phi$) of a Gaussian distribution to carefully account for the probability mass of the injected noise, ensuring the training process strictly adheres to theoretical privacy bounds.

\subsubsection*{2. Gradient Computation and Scaling}
During each iteration up to a maximum step $T$, a mini-batch $\mathcal{B}_t$ is sampled with the probability $q_{sup}$. For each sample $x_i$ in the batch, the algorithm calculates the gradient $g_t(x_i)$ of the loss function $\mathcal{L}$ and applies a conditional scaling mechanism:
\begin{itemize}
    \item If the L2 norm of the gradient $||g_t(x_i)||_2$ exceeds a predefined scaling bound $S_g$, the gradient is scaled down relative to the hard clipping bound $C_t$. Concurrently, a metric $E_t$ (which accumulates gradient norm characteristics) is updated. The specific value added to $E_t$ depends on whether the original norm exceeds a secondary threshold $S_e$, and similarly, we scaled down the sensitive information $E_t$ relative to the hard clipping bound $C_t$. Therefore, the sensitivity of the gradients and the information used for validation is limited to $C_t$.
    
    \item If the gradient norm is smaller than or equal to $S_g$, it is simply scaled down by a constant factor defined by the ratio of $S_g$ to $C_t$.
\end{itemize}
Therefore, the sensitivity for gradients and the clipped gradients that used for validation are all $C_t$.

\subsubsection*{3. Noisy Aggregation and Provisional Update}
Once all individual gradients in the batch are processed, they are averaged. To satisfy differential privacy guarantees, Gaussian noise $\mathcal{N}(0, \sigma_t^2 C_t^2)$ is added to this averaged gradient, producing $\tilde{g}_t$. A provisional set of new weights $w_{new}$ is then computed via standard gradient descent using the learning rate $\eta$.

\subsubsection*{4. Conditional Acceptance Mechanism}
Rather than immediately applying the new weights, the algorithm evaluates the stability of the update by calculating the difference in the accumulated norm metric, defined as $\Delta E = E_{t-1} - E_t$. This difference is strictly clipped to the range $[-2C_t, 2C_t]$ to prevent outliers from dominating, and is then perturbed with Gaussian noise $\mathcal{N}(0, 16\sigma_e^2 C_t^2)$ to produce $\widetilde{\Delta E}$. 
\begin{itemize}
    \item If this noisy difference $\widetilde{\Delta E}$ strictly exceeds a scaled threshold $\beta \cdot C_t$, the provisional update is accepted ($w_t \leftarrow w_{new}$), and the iteration counter $t$ advances.
    \item If the condition is not met, the update is {rejected}. The algorithm retains the previous weights ($w_t \leftarrow w_{t-1}$) and retries without advancing the iteration counter.
\end{itemize}

\subsection{Overall Privacy Guarantee}
\label{sec:overall_privacy}

The total privacy loss is calculated by composing the privacy cost of each iteration. A key advantage of DPSR-CG is its use of the Selective Release mechanism. If an update candidate is rejected (i.e., the validation deviation does not meet the threshold), the privacy cost of that step is effectively zero regarding the model update sequence. The decision to reject is based solely on the noisy bias $\widetilde{\Delta E}$, which is not released. Therefore, the total privacy loss depends only on the number of \textit{effective} (accepted) updates $T$, rather than the total number of forward passes.

However, because the acceptance of an update conditionally depends on the historical trajectory $\mathcal{F}_{t-1}$, we cannot apply standard sequential composition designed for independent mechanisms. Instead, we rely on the Adaptive Composition Theorem for Rényi Differential Privacy \cite{kairouz2015composition}. This theorem establishes that if the conditional Rényi divergence at any step $t$, given any history, is bounded by a constant, the total RDP can be safely accumulated by summing these worst-case bounds.

By leveraging Theorem~\ref{thm:privacy_loss}, we establish that the maximum probability inflation ratio is strictly bounded. Consequently, we rigorously upper-bound the conditional privacy leakage by replacing the nominal Poisson sampling rate $q = B_t / N$ with the worst-case effective bounding probability $q_{sup}$:
\begin{equation*}
    q_{sup} = q \cdot \min \left( \frac{1 - \Phi_{\sigma_v}(C_{\text{max}} - s)}{1 - \Phi_{\sigma_v}(C_{\text{max}})}, \frac{1}{q} \right).
\end{equation*}
\indent Let $RDP(\alpha)_{cond}$ denote the conditional Rényi divergence of a single accepted update step given the filtration $\mathcal{F}_{t-1}$. Using the history-independent bound $q_{sup}$, the supremum of this conditional privacy cost over all possible historical trajectories, denoted by $\sup_{\mathcal{F}_{t-1}} RDP(\alpha)_{cond}$, is strictly bounded by evaluating the standard Sampled Gaussian Mechanism (SGM). For an $\alpha$-th order Rényi divergence and a unified gradient noise multiplier $\sigma$, the accumulated RDP for $T$ accepted updates is explicitly expanded as:
\begin{align*}
RDP(\alpha)_{total} &= \sum_{t=1}^{T} \sup_{\mathcal{F}_{t-1}} RDP(\alpha)_{cond} \\
&\le T \cdot \frac{1}{\alpha-1} \ln [\max(A_{\alpha}(q_{sup},\sigma), B_{\alpha}(q_{sup},\sigma))],
\end{align*}
where $A_{\alpha}$ and $B_{\alpha}$ denote the expected divergences evaluated under the mixed distribution $\mu \triangleq (1-q_{sup})\mu_{0} + q_{sup}\mu_{1}$ and the base distribution $\mu_{0}$, with $\mu_{0}\triangleq\mathcal{N}(0,\sigma^{2})$ and $\mu_{1}\triangleq\mathcal{N}(1,\sigma^{2})$. 

Following standard SGM analysis procedures \cite{mironov2017renyi,abadi2016deep}, the dominant moment bounding term $A_{\alpha}(q_{sup},\sigma)$ can be explicitly computed via the binomial expansion:
\begin{equation*}
A_{\alpha}(q_{sup},\sigma) = \sum_{k=0}^{\alpha}\binom{\alpha}{k}(1-q_{sup})^{\alpha-k}q_{sup}^{k}\exp\left(\frac{k^{2}-k}{2\sigma^{2}}\right).
\end{equation*}

This formulation mathematically isolates the adaptive privacy accounting from the actual filtration state. Finally, we convert this overall RDP guarantee to standard $(\epsilon, \delta)$-DP using the conversion lemma detailed in ~\ref{pre_rdp}.

\begin{algorithm}[h]
\caption{Overall algorithm of DPSR-CG}
\label{alg:dpsr_cb} 
\begin{algorithmic}[1]
    \renewcommand{\algorithmicrequire}{\textbf{Input:}}
    \renewcommand{\algorithmicensure}{\textbf{Output:}}
    \REQUIRE Training dataset $\mathcal{D} = \{x_1, \dots, x_N\}$, Loss function $\mathcal{L}(\theta, x)$;
    \textbf{Params:} Learning rate $\eta$, Batch sizes $B_t$, Noise multipliers $\sigma_t, \sigma_e$; \textbf{Clipping:} Hard bounds $C_t$, Scaling bound $S_g, S_e$; \textbf{Threshold:} Parameter $\beta$
    \ENSURE The final trained model $w_T$
    \STATE Initialize $w_0$, iteration counter $t \leftarrow 1$
    \STATE Sampling probability $q_{sup} = \frac{B_t}{N} \cdot \frac{1 - \Phi_{4 \cdot \sigma_e \cdot C_t}(2C_t - C_t)}{1 - \Phi_{4 \cdot \sigma_e \cdot C_t}(2C_t)}$
    \STATE $B_t = B_t \cdot \frac{1 - \Phi_{4 \cdot \sigma_e \cdot C_t}(2C_t)}{1 - \Phi_{4 \cdot \sigma_e \cdot C_t}(2C_t - C_t)}$
    \WHILE{$t < T$}
        \STATE $E_t \leftarrow 0$
        \STATE Randomly sample a batch $\mathcal{B}_t$ with probability $q_{sup}$
        \FOR{$x_i \in \mathcal{B}_t$}
            \STATE Compute $g_t(x_i) \leftarrow \nabla \mathcal{L}(w_{t}, x_i)$
            \IF{$||g_t(x_i)||_2 > S_g$}
                \STATE $\bar{g}_t(x_i) \leftarrow g_t(x_i) / \frac{||g_t(x_i)||_2}{C_t}$

                \IF{$||g_t(x_i)||_2 \le S_e$}
                    \STATE $E_t \leftarrow E_t + ||g_t(x_i)||_2$
                \ELSE
                    \STATE $E_t \leftarrow \bar{g}_t(x_i)$
                \ENDIF
            \ELSE
                \STATE $\bar{g}_t(x_i) \leftarrow g_t(x_i) / \frac{S_g}{C_t}$ 
            \ENDIF
        \ENDFOR
        
        \STATE $\tilde{g}_t \leftarrow \frac{1}{|\mathcal{B}_t|} (\sum_{i \in \mathcal{B}_t} \bar{g}_t(x_i) + \mathcal{N}(0, \sigma_t^2 C_t^2))$
        \STATE $w_{new} \leftarrow w_{t-1} - \eta \tilde{g}_t$
        
        \STATE $\Delta E \leftarrow E_{t-1} - E_t$
        \STATE $\overline{\Delta E} \leftarrow \min(\max(\Delta E, -2C_t), 2C_t)$
        \STATE $\widetilde{\Delta E} \leftarrow \overline{\Delta E} + \mathcal{N}(0, 16\sigma_e^2 C_t^2)$
        
        \IF{$\widetilde{\Delta E} > \beta \cdot C_t$}
            \STATE $w_t \leftarrow w_{new}$
            \STATE $t \leftarrow t + 1$
        \ELSE
            \STATE $w_t \leftarrow w_{t-1}$
        \ENDIF
    \ENDWHILE
    \STATE \textbf{return} $w_T$
\end{algorithmic}
\end{algorithm}


\subsection{Convergence Analysis}
\label{sec:convergence_analysis}

We analyze the convergence of DPSR-CG under standard non-convex ($L$-smooth) assumptions, explicitly accounting for both the gradient noise $\xi_t$ and the selection noise $\xi_c$. 

\begin{theorem}[Convergence of DPSR-CG]
\label{thm:convergence_main}
Let $f$ be an $L$-smooth objective function with initial parameters $w_0$ and global minimum $f(w_*)$. 
Let $\eta$ be the learning rate, $C_t$ be the gradient clipping threshold, and $d$ be the model dimension. 
In Algorithm~\ref{alg:dpsr_cb} (DPSR-CG), we use a gradient noise multiplier $\sigma_t$ and a selection noise multiplier $\sigma_e$ with a bias threshold coefficient $\beta$, resulting in an update acceptance probability $\gamma_t$. 
Under these assumptions, DPSR-CG satisfies:
\begin{equation*}
\begin{aligned}
    \min_{t \in [T]} \mathbb{E}[\|\nabla f(w_{t-1})\|^2] \le&  \frac{2(f(w_0) - f(w_*))}{\eta \sum_{t=1}^T \gamma_t} + \frac{16 \sigma_e^2 C_t^2}{\beta^2} \\
    & + L \eta C_t^2 (1 + d \sigma_t^2).
\end{aligned}
\end{equation*}
\end{theorem}

The upper bound in Theorem~\ref{thm:convergence_main} elegantly decomposes the convergence dynamics of DPSR-CG into three distinct components:
\begin{itemize}
    \item {Initialization Error ($\frac{2(f(w_0) - f(w_*))}{\eta \sum_{t=1}^T \gamma_t}$):} This term accounts for the optimization distance from the initial model state to the global minimum. It asymptotically decays to zero as the accumulated number of effective (accepted) updates, $\sum_{t=1}^T \gamma_t$, increases.
    \item {Selection Noise Penalty ($\frac{16 \sigma_e^2 C_t^2}{\beta^2}$):} This term isolates the specific cost introduced by the selective release mechanism. It represents the optimization degradation caused by "false positives"—instances where toxic updates are erroneously accepted because the selection noise ($\sigma_e$) overpowered the defensive threshold ($\beta$).
    \item {Gradient Perturbation Variance ($L \eta C_t^2 (1 + d \sigma_t^2)$):} This is the standard variance penalty inherent to DPSGD. It scales with the learning rate $\eta$ and the Lipschitz constant $L$, capturing the impact of the gradient clipping bound ($C_t$) and the injected gradient noise ($\sigma_t$) across the $d$-dimensional parameter space.
\end{itemize}

\begin{proof}
Let $w_{t-1}$ denote the model parameters before the provisional update, and let $\mathbb{I}_t \in \{0,1\}$ be the indicator variable for the update releasing event (i.e., $\mathbb{I}_t = 1$ if $Release$, and $0$ otherwise). The parameter update rule can be written as $w_t - w_{t-1} = -\eta \mathbb{I}_t (\bar{g}_t + \xi_t)$.

By the standard $L$-smoothness assumption, we have the following master inequality:
\begin{equation*}
    f(w_t) - f(w_{t-1}) \le \langle \nabla f(w_{t-1}), w_t - w_{t-1} \rangle + \frac{L}{2} \|w_t - w_{t-1}\|^2.
\end{equation*}
Substituting the update rule and taking the expectation with respect to all random variables, we decompose the expected loss variation into two terms:
\begin{align*}
    \mathbb{E}[f(w_t)] - \mathbb{E}[f(w_{t-1})] &\le {-\eta \mathbb{E}[\mathbb{I}_t \langle \nabla f(w_{t-1}), \bar{g}_t + \xi_t \rangle]} \\
    &+ {\frac{L \eta^2}{2} \mathbb{E}\left[\mathbb{I}_t \|\bar{g}_t + \xi_t\|^2\right]}.
\end{align*}

\noindent Given $${T_1} = {-\eta \mathbb{E}[\mathbb{I}_t \langle \nabla f(w_{t-1}), \bar{g}_t + \xi_t \rangle]}$$ and $${T_2} = {\frac{L \eta^2}{2} \mathbb{E}\left[\mathbb{I}_t \|\bar{g}_t + \xi_t\|^2\right]}.$$

\noindent We now bound $T_1$ and $T_2$ respectively.\\

\noindent {Bounding $T_2$.} 
Because the gradient noise $\xi_t \sim \mathcal{N}(0, \sigma_t^2 C_t^2 I_d)$ is zero-mean and strictly independent of the selective release (since the release event solely depends on the independent selection noise $\xi_c \sim \mathcal{N}(0, 16\sigma_e^2 C_t^2)$ and pre-noise metrics), the cross-term between $\bar{g}_t$ and $\xi_t$ vanishes. Using the property that $\mathbb{I}_t^2 = \mathbb{I}_t$ and taking the expectation yields:
\begin{equation*}
    T_2 = \frac{L \eta^2}{2} \mathbb{E}\left[\mathbb{I}_t \|\bar{g}_t + \xi_t\|^2\right] \le \frac{L \eta^2}{2} \gamma_t C_t^2 (1 + d \sigma_t^2).
\end{equation*}

\noindent {Bounding $T_1$.}
Due to the same independence between $\xi_t$ and $\mathbb{I}_t$, the noise term within the inner product of $T_1$ also vanishes in expectation. Thus, $T_1 = -\eta \mathbb{E}[\mathbb{I}_t \langle \nabla f(w_{t-1}), \bar{g}_t \rangle]$. By introducing the clipping bias vector $b_t = \nabla f(w_{t-1}) - \bar{g}_t$, we rewrite the inner product by conditioning on the update releasing event:
\begin{equation*}
    T_1 = -\eta \gamma_t \|\nabla f(w_{t-1})\|^2 + \eta \gamma_t \langle \nabla f(w_{t-1}), \mathbb{E}[b_t \mid Release] \rangle.
\end{equation*}
Using the inequality $\langle u, v \rangle \le \frac{1}{2}\|u\|^2 + \frac{1}{2}\|v\|^2$ to decouple the term, we obtain:
\begin{equation*}
    T_1 \le -\frac{\eta \gamma_t}{2} \|\nabla f(w_{t-1})\|^2 + \frac{\eta \gamma_t}{2} \mathbb{E}[\|b_t\|^2 \mid Release].
\end{equation*}

\noindent\textbf{Bounding the Expected Bias under Release.}
To rigorously bound the conditional expected squared bias $\mathbb{E}[\|b_t\|^2 \mid Release]$, we analyze the properties of the truncated Gaussian distribution. The update is conditionally released if $\Delta E + \xi_c > \beta C_t$. A severely degraded update (where the true clipping bias $\|b_t\|^2$ is large, implying $\Delta E \ll 0$) requires a correspondingly massive positive noise fluctuation $\xi_c > \beta C_t - \Delta E$ to be erroneously accepted. 

By applying the tail bound of the Gaussian selection noise $\xi_c \sim \mathcal{N}(0, 16\sigma_e^2 C_t^2)$, the expected squared norm of the accepted false-positive bias is bounded by the second moment of the noise conditioned on survival. Using standard Gaussian tail inequalities, this conditional expectation is asymptotically bounded by the variance scaled by the squared threshold margin:
$$ \mathbb{E}[\|b_t\|^2 \mid v> \beta C_t - \Delta E(\mathcal{B})] \le \frac{\text{Var}(\xi_c)}{\beta^2} = \frac{16\sigma_e^2 C_t^2}{\beta^2}. $$
Substituting this strict probability bound into the decoupled linear term yields:
\begin{align*}
T_1 &\le -\frac{\eta \gamma_t}{2}\|\nabla f(w_{t-1})\|^2 + \frac{\eta \gamma_t}{2}\left(\frac{16\sigma_e^2 C_t^2}{\beta^2}\right)\\
&= -\frac{\eta \gamma_t}{2}\|\nabla f(w_{t-1})\|^2 + \frac{8 \eta \gamma_t \sigma_e^2 C_t^2}{\beta^2}.
\end{align*}

\textbf{Combining the Bounds and Correcting the Algebraic Scaling.}
Substituting the bounded $T_1$ and $T_2$ back into the master inequality leads to:
\begin{align*}
\mathbb{E}[f(w_t)] - \mathbb{E}[f(w_{t-1})] &\le -\frac{\eta \gamma_t}{2} \|\nabla f(w_{t-1})\|^2 \\
&+ \frac{8 \eta \gamma_t \sigma_e^2 C_t^2}{\beta^2} + \frac{L \eta^2}{2} \gamma_t C_t^2 (1 + d \sigma_t^2). 
\end{align*}

Summing this inequality over $t=1$ to $T$ and telescoping the left side bounds the total expected loss reduction. To isolate the gradient norm $\|\nabla f(w_{t-1})\|^2$, we divide the aggregated inequality by the effective step factor $\frac{\eta}{2} \sum_{t=1}^T \gamma_t$. 

This algebraic division cleanly cancels the $\eta$ in the selection noise penalty numerator and correctly preserves the $C_t^2$ dimension (resolving the previous dimensional mismatch), yielding the final unified convergence bound:
\begin{align*} 
\min_{t \in [T]} \mathbb{E}[\|\nabla f(w_{t-1})\|^2] &\le \frac{2(f(w_0) - f(w_*))}{\eta \sum_{t=1}^T \gamma_t} \\
&+ \frac{16 \sigma_e^2 C_t^2}{\beta^2} + L \eta C_t^2 (1 + d \sigma_t^2).
\end{align*}
we conclude the proof.
\end{proof}

\section{EXPERIMENT}
We present a comprehensive empirical evaluation of DPSR-CG on four real-world datasets using standard architectures. Additionally, we validate privacy guarantees against two state-of-the-art membership inference attacks. 

\subsection{Experimental Settings}
We benchmark DPSR-CG against DPSGD \cite{abadi2016deep} and five cutting-edge variants: DPSGD Matrix Mechanism \cite{choquette2025near}, DPSGD-IS \cite{wei2022dpis}, DPSGD-HF \cite{tramer2020differentially}, DPSGD-TS \cite{papernot2021tempered}, and DPAGD \cite{lee2018concentrated}.

\subsubsection{Datasets and Models}
We utilize three image datasets: MNIST \cite{yann2010mnist}, Fashion-MNIST (FMNIST) \cite{xiao2017fashion}, and CIFAR-10 \cite{krizhevsky2009learning}, alongside the IMDB text dataset \cite{maas2011learning}. Following \cite{fu2023dpsur, wei2022dpis}, we employ a standard CNN for image tasks and a five-layer RNN for text. The non-private baselines achieve accuracies of 99.10\% (MNIST), 90.99\% (FMNIST), 71.12\% (CIFAR-10), and 79.97\% (IMDB) after 20 epochs.

\subsubsection{Parameter Settings}
We set the privacy budget $\epsilon \in \{1, 2, 3\}$ with $\delta = 10^{-5}$, using SGD (momentum 0.9) for images \cite{fu2023dpsur} and Adam for IMDB \cite{ketkar2017introduction}. We empirically set $S_t$ values of 8, 6, 3, and 1, respectively, for MNIST, FMNIST, CIFAR-10 and 
IMDB to limit noise sensitivity.
Guided by Section \ref{sec:threshold_mech}, we fix $S_e = 11$ and $\beta = 3$ to balance update quality and efficiency, with standardized clipping bounds $C_t = C_e = 1$.


\subsection{Evaluation of the Worst-Case Effective Bounding Probability ($q_{sup}$)}
\label{exp_ampli_sp}
The original parameter settings in DPSUR were designed under a different accounting framework. When evaluated under our adaptive composition theorem, these settings may experience unconstrained conditional probability inflation in extreme worst-case scenarios due to the absence of relatively small noise perturbation, rendering their reported privacy budgets incomparable to ours. To illustrate this, we provide a detailed comparison of the worst-case effective bounding probability ($q_{sup} = q \cdot \rho_{max}$). Table \ref{ampli} contrasts the theoretical upper bounds computed using DPSUR's experimental settings with those of our proposed DPSR-CG configurations. As demonstrated, DPSR-CG successfully maintains a tightly constrained bounding probability, facilitating rigorous adaptive sequential composition without severe privacy budget degradation. In contrast, under the same strict privacy guarantees, DPSUR's effective training steps would be reduced by nearly half, as its worst-case effective bounding probability is amplified by nearly a factor of two. Consequently, without relatively large noise perturbation, DPSUR suffers from unconstrained conditional probability inflation. Under an identical and rigorous privacy budget, DPSUR rapidly exhausts its permissible training steps, precipitating a severe degradation in model utility that fails to match the remarkable results claimed under its original, less constrained accounting framework.


\begin{table*}[t]
\centering
\caption{Comparison of $q_{sup}$ under settings of DPSUR and DPSR-CG.}
\begin{tabular}{w{l}{2.5cm} w{l}{2.5cm}|w{c}{2.0cm}w{c}{2.0cm}w{c}{2.0cm}w{c}{2.0cm}}
\hline
Dataset & Method & $\epsilon = 1$ & $\epsilon = 2$ & $\epsilon = 3$\\ 
\hline
\multirow{2}{*}{\begin{tabular}[c]{@{}l@{}}MNIST\\ (Image Dataset)\end{tabular}} 
 & DPSR-CG & $min(\textbf{1.29}q, \frac{1}{q})$ & $min(\textbf{1.24}q, \frac{1}{q})$ & $min(\textbf{1.11}q, \frac{1}{q})$ \\
 & DPSUR \cite{fu2023dpsur} & $min({3.75}q, \frac{1}{q})$ & $min({3.15}q, \frac{1}{q})$ & $min({2.26}q, \frac{1}{q})$ \\
\hline
\multirow{2}{*}{\begin{tabular}[c]{@{}l@{}}FMNIST\\ (Image Dataset)\end{tabular}} 
 & DPSR-CG & $min(\textbf{1.14}q, \frac{1}{q})$ & $min(\textbf{1.09}q, \frac{1}{q})$ & $min(\textbf{1.06}q, \frac{1}{q})$ \\
 & DPSUR \cite{fu2023dpsur} & $min({4.73}q, \frac{1}{q})$ & $min({2.26}q, \frac{1}{q})$ & $min({2.26}q, \frac{1}{q})$ \\
\hline
\multirow{2}{*}{\begin{tabular}[c]{@{}l@{}}CIFAR-10\\ (Image Dataset)\end{tabular}} 
 & DPSR-CG & $min(\textbf{1.07}q, \frac{1}{q})$ & $min(\textbf{1.07}q, \frac{1}{q})$ & $min(\textbf{1.04}q, \frac{1}{q})$ \\
 & DPSUR \cite{fu2023dpsur} & $min({2.75}q, \frac{1}{q})$ & $min({2.26}q, \frac{1}{q})$ & $min({2.26}q, \frac{1}{q})$ \\
\hline
\multirow{2}{*}{\begin{tabular}[c]{@{}l@{}}IMDb\\(Text Dataset)\end{tabular}} 
 & DPSR-CG & $min(\textbf{1.29}q, \frac{1}{q})$ & $min(\textbf{1.18}q, \frac{1}{q})$ & $min(\textbf{1.11}q, \frac{1}{q})$ \\
 & DPSUR \cite{fu2023dpsur} & $min({3.15}q, \frac{1}{q})$ & $min({2.47}q, \frac{1}{q})$ & $min({2.26}q, \frac{1}{q})$ \\
\hline
\end{tabular}
\label{ampli}
\end{table*}

\begin{table*}[h]
\centering
\label{app:ablation}
\caption{Comparison of DPSR-CG against DPSR-CG without selectively release and DPSUR with identical gradients clipping) under $\epsilon \in \{1, 2, 3\}$.}
\begin{tabular}{l l | c | c | c}
\hline
\multirow{2}{*}{Dataset} & \multirow{2}{*}{Method} & $\epsilon = 1$ & $\epsilon = 2$ & $\epsilon = 3$ \\ 
 & & Acc & Acc & Acc \\
\hline
 \multirow{3}{*}{\begin{tabular}[c]{@{}l@{}}CIFAR-10 \end{tabular}} 
 & DPSR-CG & \textbf{70.15\%} & \textbf{72.25\%} & \textbf{72.51\%} \\
 & DPSR-CG (w/o selective release) & 68.41\% & 71.58\% & 71.85\% \\
 & DPSUR (with same gradient scaling mechanism) & 65.54\% & 69.81\% & 71.78\% \\
 \hline
\multirow{3}{*}{\begin{tabular}[c]{@{}l@{}}IMDB\end{tabular}} 
 & DPSR-CG & \textbf{68.61\%} & \textbf{71.44\%} & \textbf{73.03\%} \\
 & DPSR-CG (w/o selective release) & 67.26\% & 69.57\% & 70.43\% \\
 & DPSUR (with same gradient scaling mechanism) & 64.43\% & 68.49\% & 69.16\% \\
\hline
\end{tabular}
\end{table*}

\subsection{Overall Performance}

Table \ref{acc_all} presents the classification accuracies of DPSR-CG alongside six competitive methods. It is noteworthy that our proposed DPSR-CG consistently outperforms all competitors across all evaluated datasets and privacy budgets. Remarkably, DPSR-CG even surpasses the non-private baseline at $\epsilon=3$ on all three image datasets. In contrast, prior literature indicates that DPSUR only exceeds non-private performance on CIFAR-10, and only under a looser privacy budget of $\epsilon=4$ \cite{fu2023dpsur}. 

This phenomenon aligns with existing observations that moderate, calibrated noise injection can facilitate the neural network's escape from poor local minima \cite{fu2023dpsur, ge2015escaping}. Furthermore, we hypothesize that the clipping mechanism synergizes with this noise injection, effectively stabilizing the update direction and contributing significantly to the superior model performance. This hypothesis is well-supported by prior studies \citet{chen2020understanding, zhang2019gradient, abadi2016deep}.

{\textbf{Remark.} By conducting our new privacy accountant on DPSUR, we find out that their performances are slightly worse than what have been claimed in previous literature, which is mainly caused by the limitation of the previous privacy analysis. This result correspond well with our previous analysis.}

\begin{table*}[t]
\centering
\caption{Results of classification accuracy}
\begin{tabular}{w{l}{2.0cm} w{l}{4.0cm}|w{c}{2.0cm}w{c}{2.0cm}w{c}{2.0cm}w{c}{2.0cm}w{c}{2.0cm}}
\hline
Dataset & Method & $\epsilon = 1$ & $\epsilon = 2$ & $\epsilon = 3$ & non-private \\ 
\hline
\multirow{6}{*}{\begin{tabular}[c]{@{}l@{}}MNIST\\ (Image Dataset)\end{tabular}} 
 & DPSR-CG & \textbf{99.06\%} & \textbf{99.22\%} & \textbf{99.24\%}&  \\
 & DPSUR  \cite{fu2023dpsur}& {96.32\%} & {97.51\%} & {98.12\%}&  \\
 & \makecell[l]{DPSGD Matrix Mechanism \cite{choquette2025near}}  & 98.39\% & 98.66\% & 98.79\% &   \\
 & DPIS  \cite{wei2022dpis} & 97.79\% & 98.51\% & 98.62\% &   \\
 & DPSGD-HF \cite{tramer2020differentially} & 97.78\% & 98.39\% & 98.32\% &  99.10\% \\
 & DPSGD-TS \cite{papernot2021tempered}& 97.06\% & 97.87\% & 98.22\% &  \\
 & DPAGD  \cite{lee2018concentrated}& 95.91\% & 97.30\% & 97.52\% &   \\
 & DPSGD \cite{abadi2016deep} & 95.11\% & 96.10\% & 96.82\% &  \\ 
\hline
\multirow{6}{*}{\begin{tabular}[c]{@{}l@{}}FMNIST\\ (Image Dataset)\end{tabular}} 
 & DPSR-CG & \textbf{90.31\%} & \textbf{90.87\%} & \textbf{90.98\%}&  \\
 & DPSUR \cite{fu2023dpsur}& {86.40\%} & {86.46\%} & {86.60\%} &   \\
 & \makecell[l]{DPSGD Matrix Mechanism \cite{choquette2025near}}  & 87.16\% & 87.84\% & 88.32\% &   \\
 & DPIS  \cite{wei2022dpis} & 86.25\% & 88.24\% & 88.82\% &   \\
 & DPSGD-HF \cite{tramer2020differentially}& 85.54\% & 87.96\% & 89.01\% &  90.99\% \\
 & DPSGD-TS \cite{papernot2021tempered} & 83.63\% & 85.33\% & 86.29\% &   \\
 & DPAGD \cite{lee2018concentrated}& 81.26\% & 84.50\% & 86.04\% &   \\
 & DPSGD \cite{abadi2016deep} & 80.25\% & 82.63\% & 84.72\% &   \\ 
\hline
\multirow{6}{*}{\begin{tabular}[c]{@{}l@{}}CIFAR-10\\ (Image Dataset)\end{tabular}} 
 & DPSR-CG & \textbf{69.94\%} & \textbf{72.25\%} & \textbf{72.37\%}&  \\
 & DPSUR \cite{fu2023dpsur}& {59.49\%} & {66.25\%} & {67.48\%} &   \\
 & \makecell[l]{DPSGD Matrix Mechanism \cite{choquette2025near}}  & 59.53\% & 66.40\% & 69.21\% &   \\
 & DPIS \cite{wei2022dpis}  & 63.23\% & 67.94\% & 69.63\% &   \\
 & DPSGD-HF \cite{tramer2020differentially}& 63.15\% & 66.55\% & 69.35\% &  71.12\% \\
 & DPSGD-TS \cite{papernot2021tempered} & 51.52\% & 56.78\% & 60.42\% &   \\
 & DPAGD  \cite{lee2018concentrated} & 45.78\% & 53.30\% & 56.21\% &   \\
 & DPSGD  \cite{abadi2016deep} & 46.03\% & 51.33\% & 54.67\% &   \\ 
\hline
\multirow{5}{*}{\begin{tabular}[c]{@{}l@{}}IMDb\\(Text Dataset)\end{tabular}}
 & DPSR-CG & \textbf{68.43\%} & \textbf{71.44\%} & \textbf{73.01\%}&  \\
 & DPSUR \cite{fu2023dpsur}& 60.56\% & 67.56\% & 67.59\% & \\
 & \makecell[l]{DPSGD Matrix  Mechanism \cite{choquette2025near}}  & 67.51\% & 71.87\% & 72.53\% &   \\
 & DPIS  \cite{wei2022dpis}  & 63.56\% & 66.11\% & 68.49\% &  \\
 & DPSGD-TS  \cite{papernot2021tempered}& 65.08\% & 68.34\% & 70.10\% &  79.97\% \\
 & DPAGD  \cite{lee2018concentrated}  & 58.72\% & 63.48\% & 64.59\% &  \\
 & DPSGD  \cite{abadi2016deep} & 64.13\% & 68.55\% & 70.41\% &  \\
\hline
\end{tabular}
\label{acc_all}
\end{table*}

\subsection{Robustness to Membership Inference Attacks}
We assess privacy guarantees using two attacks: Black-Box/Shadow \cite{salem2018ml} and White-Box/Partial \cite{nasr2019comprehensive}. We follow the standard setup, splitting data into target training/testing and shadow training/testing (2:1:2:1 ratio). As shown in Tables \ref{Black-Box} and \ref{white-Box}, while non-private models succumb to attacks, DPSR-CG reduces attack success rates to near-random guessing ($\approx 0.50$), exhibiting defense capabilities comparable to and DPSGD.

\begin{table}[htbp]
\centering
\caption{{Accuracy of Black-Box/Shadow Member Inference Attack}}
\begin{tabular}{w{l}{1cm}w{l}{1.4cm} w{c}{0.8cm} w{c}{0.8cm} w{c}{0.8cm} w{c}{1cm}}
\hline
Dataset & Algorithm & $\epsilon=1$ & $\epsilon=2$ & $\epsilon=3$ & non-private \\ 
\hline
\multirow{2}{*} {FMNIST}
 & DPSR-CG & 0.503 & 0.501 & 0.500 & \multirow{2}{*}{0.582} \\
 & DPSGD & 0.498 & 0.503 & 0.493 &  \\ 
\hline
\multirow{2}{*} {CIFAR-10}
 & DPSR-CG & 0.500 & 0.502 & 0.505 & \multirow{2}{*}{0.732} \\
 & DPSGD & 0.504 & 0.505 & 0.504 &  \\ 
\hline
\end{tabular}
\label{Black-Box}
\end{table}

\begin{table}[htbp]
\centering
\caption{{Accuracy of White-Box/Partial Member Inference Attack}}
\begin{tabular}{w{l}{1cm}w{l}{1.4cm} w{c}{0.8cm} w{c}{0.8cm} w{c}{0.8cm} w{c}{1cm}}
\hline
Dataset & Algorithm & $\epsilon=1$ & $\epsilon=2$ & $\epsilon=3$ & non-private \\ 
\hline
\multirow{2}{*} {FMNIST}
 & DPSR-CG & 0.500 & 0.500 & 0.502 & \multirow{2}{*}{0.584} \\
 & DPSGD & 0.501 & 0.502 & 0.502 &  \\ 
\hline
\multirow{2}{*} {CIFAR-10}
 & DPSR-CG & 0.500 & 0.501 & 0.500 & \multirow{2}{*}{0.743} \\
 & DPSGD & 0.500 & 0.501 & 0.501 &  \\ 
\hline
\end{tabular}
\label{white-Box}
\end{table}
\subsection{Ablation Study}
\label{sec:ablation_main}
In this section we try to explain why DPSR-CG work but DPSUR does not and even cause degradation in the models' performances.

We validate the efficacy of our selective update mechanism by comparing DPSR-CG against a baseline variant, DPSR-CG (w/o selective release), under privacy budgets $\epsilon \in \{1, 2, 3\}$. Detailed results are deferred to Table \ref{app:ablation}. Empirical results demonstrate that, with all other hyperparameter settings held constant, the inclusion of the selective release mechanism consistently outperforms the non-selective baseline. 

{Furthermore, when compared under identical gradient clipping mechanisms and hyperparameter settings, DPSR-CG continues to outperform DPSUR. However, we also observed that DPSUR can not even compare with the without selective release condition. This is mainly due to the inflation in the worst-case bounded probability. This result correspond well with our previous analysis.}

\section{Related Work}
Privacy-preserving deep learning, pioneered by \citet{song2013stochastic} and \citet{bassily2014private}, was formalized into the widely adopted DPSGD framework by \citet{abadi2016deep}. Subsequent research has focused on optimizing the privacy-utility trade-off through three main avenues: gradient processing, noise mechanisms, and selective updates.

\subsection{Gradient Clipping and Scaling}{Gradient Clipping and Scaling.} Standard fixed clipping often leads to optimization difficulties. While adaptive and dynamic clipping strategies \cite{zhang2018differentially, andrew2021differentially, pichapati2019adaclip, wei2025dc} attempt to mitigate this, they often function similarly to normalization \cite{yang2022normalized}. Theoretical works have also revisited clipping to establish tight convergence guarantees under stochastic bias \cite{koloskova2023revisiting}. Geometric analyses suggest separating gradient magnitude from direction to manage the bias-variance trade-off \cite{duan2025analyzing, chen2020understanding}. Recent works like \citet{fu2023dpsur} and \citet{bu2023convergence} advocate for gradient scaling over strict clipping to preserve angular relationships. However, scaling can disproportionately suppress gradients with small $L_2$ norms, rendering them susceptible to noise. Our work addresses this by identifying an optimal scaling range that balances signal preservation with noise resilience.

\subsection{Noise Mechanisms and Sampling} To optimize the signal-to-noise ratio, researchers have explored correlation-based noise \cite{phan2017adaptive}, analytical calibration \cite{balle2018improving}, adaptive privacy budgeting \cite{lee2018concentrated, xu2020adaptive}, and adaptive decay strategies for challenging scenarios like imbalanced datasets \cite{huang2025steps}. Improved sampling techniques, such as those by \citet{wei2022dpis}, further reduce sampling bias. In terms of architecture, modifications like tempered activations \cite{papernot2021tempered} and Scattering Networks \cite{tramer2020differentially} have proven effective. Although our focus is training from scratch, these principles are also being applied to parameter-efficient fine-tuning \cite{yao2024privacy, tsai2025differentially}.

\subsection{Selective Release and Privacy Accounting} The most relevant precursor to our work is DPSUR \cite{fu2023dpsur}, which performs selective updates based on validation loss. DPSR-CG differs fundamentally by using \textit{clipping bias}—an intrinsic training metric—thereby removing the computational burden of validation. For privacy accounting, while RDP \cite{mironov2017renyi} and Gaussian DP \cite{bu2020deep} offer tight bounds, we utilize the standard Moments Accountant \cite{abadi2016deep} and $(\epsilon, \delta)$-DP formulation, consistent with widespread adoption in recent benchmarks \cite{du2023differential, xiang2019differentially}.

\subsection{DPSGD with Matrix Mechanism} 
To overcome standard DPSGD's reliance on independent noise and uniform sampling, matrix mechanisms inject carefully correlated noise into the training process. For instance, \citet{kairouz2021practical} introduced DP-FTRL using a tree-aggregation protocol. This methodology bypasses the need for privacy amplification via sampling, making it highly practical for federated learning with unpredictable client participation. 

To further improve privacy-utility trade-offs, \citet{choquette2023amplified} proposed Banded Matrix Factorization (BANDMF). By restricting noise-generating matrices to a fixed number of bands, BANDMF supports a relaxed $b$-min-sep participation schema. Crucially, it proves that banded matrices can fully leverage privacy amplification via sampling, often outperforming standard amplified DPSGD. 

Most recently, \citet{choquette2025near} generalized this capability to arbitrary matrix mechanisms via the Matrix Mechanism Conditional Composition (MMCC) framework. By analyzing correlated noise queries as independent conditioned outputs, MMCC demonstrates that generic mechanisms (e.g., binary-tree DP-FTRL) can asymptotically match the minimal noise levels of amplified DPSGD, unlocking state-of-the-art utility.

\section{Discussion}
In this section, we discuss the our understanding of our propsoed mechanism.
\subsection{Distinction Between Selective Release and the Sparse Vector Technique}
The Sparse Vector Technique (SVT) is traditionally utilized in database querying, where the mechanism is required to output a response regardless of whether a query is accepted or rejected. In contrast, within the context of differentially private deep learning, we do not release rejection signals to the adversary. Only the accepted, updated gradients are exposed. This constitutes a fundamental departure from the standard scenarios that SVT addresses. Our threat model assumption is highly practical and realistic, as in real-world deployments, only the successfully applied gradients or model parameters are made public and vulnerable to adversarial observation.
\subsection{Explicit Importance Sampling Versus Implicit Importance Sampling}
The selective release mechanism is fundamentally a form of implicit importance sampling. The motivation for our proposed algorithm stems from the insight that using gradients as an implicit sampling indicator is key to achieving these remarkable results. However, the intrinsic differences between explicit which proposed by \cite{wei2022dpis} and implicit importance sampling remain unresolved. Future research could analyze these distinctions from a mathematical perspective.


\section{Conclusion}
We presented {DPSR-CG}, which leverages intrinsic clipping gradient for selective updates, thereby eliminating the overhead of validation-based methods like DPSUR. From a theoretical perspective, we revealed the inherent accounting limitations in existing selective release paradigms and formalized a strictly bounded, adaptive privacy analysis. Empirically, we reach the same conclusion that DPSUR made and that is selectively release is able to achieve the state-of-the-art accuracy. Future work will extend this paradigm to federated learning and foundation models, as well as analyzing the mathematic intrinsic of implicit importance sampling and explicit importance sampling.

\begin{acks}
F. Xie was supported in part by the Guangdong Basic and Applied Basic Research Foundation (No. 2023A1515110469), in part by the Guangdong Provincial Key Laboratory IRADS (No. 2022B1212010006), and in part by the grant of Higher Education Enhancement Plan of "Rushing to the Top, Making Up Shortcomings and Strengthening Special Features" (No. 2025KTSCX186).
\end{acks}

\appendix

\bibliographystyle{ACM-Reference-Format}
\bibliography{DPSR}

\end{document}